\documentclass{article}
\usepackage{arxiv}
\usepackage[utf8]{inputenc}
\usepackage{color}
\usepackage{amsmath,amssymb,amsthm}
\usepackage{mathtools,dsfont}

\usepackage{multirow, rotating}
\usepackage{xspace}

\usepackage
[vlined,ruled,titlenotnumbered]
{algorithm2e}
\usepackage{xcolor, soul}
\usepackage{colortbl}
\sethlcolor{yellow}
\newcommand{\cmmnt}[1]{}
\sethlcolor{lightgray}
\usepackage{xcolor}   
\newcommand{\swgsemo}{SW-GSEMO\xspace}
\newcommand{\slidinggsemo}{\swgsemo}
\usepackage{color} 
\usepackage{dblfloatfix}
\usepackage{xspace,dsfont}

\newcommand{\ie}{i.\,e.\xspace}
\newcommand{\eg}{e.\,g.\xspace}

\newcommand{\temax}{t_{\mathit{max}}}

\newcommand{\OPT}{\mathrm{OPT}}
\newcommand{\R}{\mathbb{R}}
\newcommand{\N}{\mathbb{N}}

\newcommand{\gsemo}{GSEMO\xspace}

\usepackage{algpseudocode,booktabs,amsmath}

\title{Fast Pareto Optimization Using Sliding Window Selection}
\author{Frank Neumann\\
Optimisation and Logistics\\
School of Computer and Mathematical Sciences\\
The University of Adelaide\\
Adelaide, Australia
\And
Carsten Witt\\
Algorithms, Logic and Graphs\\
DTU Compute\\ Technical University of Denmark\\
2800 Kgs. Lyngby Denmark
}

\usepackage{graphicx}
\newtheorem{definition}{Definition}

\newtheorem{theorem}{Theorem}

\begin{document}
\maketitle

\begin{abstract}
Pareto optimization using evolutionary multi-objective algorithms has been widely applied to solve constrained submodular optimization problems. A crucial factor determining the runtime of the used evolutionary algorithms to obtain good approximations is the population size of the algorithms which grows with the number of trade-offs that the algorithms encounter.
In this paper, we introduce a sliding window speed up technique for recently introduced algorithms. We prove that
our technique eliminates the population size as a crucial factor negatively impacting the runtime and achieves the same theoretical performance guarantees as previous approaches within less computation time. Our experimental investigations for the classical maximum coverage problem confirms that our sliding window technique clearly leads to better results for a wide range of instances and constraint settings.
\end{abstract}

\section{Introduction}

Many real-world optimization problems face diminishing returns when adding additional elements to a solution and can be formulated in terms of a submodular function~\cite{NemhauserWolsey78,DBLP:books/cu/p/0001G14}. Problems that can be stated in terms of a submodular function include classical combinatorial optimization problems such as the computation of a maximum coverage~\cite{DBLP:journals/ipl/KhullerMN99} or maximum cut~\cite{10.1145/227683.227684} in graphs as well as regression problems arising in machine learning.

Classical approximation algorithms for monotone submodular optimization problems under different types of constraints rely on greedy approaches which select elements with the largest benefit/cost gain according to the gain with respect to the given submodular function and the additional cost with respect to the given constraint~\cite{ZhangVorobeychickAAAI16,DBLP:books/cu/p/0001G14}.
During the last years, evolutionary multi-objective algorithms have been shown to provide the same theoretical worst case performance guarantees on 
solution quality as the greedy approaches while clearly outperforming classical greedy approximation algorithms in practice~\cite{FriedrichNeumannECJ15,DBLP:conf/nips/QianYZ15,QianSYTIJCAI17,DBLP:conf/ppsn/NeumannN20}. Results are obtained by means of rigorous runtime analysis (see \cite{DBLP:books/daglib/0025643,BookDoeNeu} for comprehensive overviews) which is a major tool for the analysis of evolutionary algorithms.

The analyzed approaches use a 2-objective formulation of the problem where the first objective is to maximize the given submodular function and the second objective is to minimize the cost of the given constraint.
The two objectives are optimized simultaneously using variants of the \gsemo algorithm~\cite{DBLP:journals/tec/LaumannsTZ04,1299908} and the algorithm outputs the feasible solution with the highest function value when terminating. Before \gsemo has been applied analyzed for submodular problems in the prescribed way, such a multi-objective set up has already been shown to be provably successful for other constrained single-objective optimization problems such as minimum spanning trees~\cite{DBLP:journals/nc/NeumannW06} and vertex and set covering problems~\cite{DBLP:journals/ec/FriedrichHHNW10,DBLP:journals/algorithmica/KratschN13} through different types of runtime analyses.

A crucial factor which significantly influences the runtime of these Pareto optimization approaches using \gsemo is the size of the population that the algorithm encounters.
This is in particular the case for problems where both objectives can potentially attain exponentially many different functions values. In the context of submodular optimization this is the case iff the function $f$ to be optimized as well as the constraint allow exponentially many trade-offs.
In order to produce a new solution, a solution chosen uniformly at random from the current population is mutated to produce an offspring. A large population size implies that often time is wasted by not selecting the right individual for mutation. 

\subsection{Our contribution}
In this paper, we present a sliding window approach for fast Pareto optimization called \swgsemo. This approach is inspired 
by the fact that several previous runtime
analyses of \gsemo for Pareto 
optimization (\eg, \cite{FriedrichNeumannECJ15, QianSYTIJCAI17}) 
consider improving steps 
of a certain kind only. More precisely,
they follow an inductive approach, 
\eg, 
where high-quality solutions 
having a cost value of~$i$ in the
constraint 
are mutated 
to high-quality solutions of 
constraint cost value~$i+1$, which 
we call a success. (More 
general sequences of larger increase 
per success 
are possible and will
 be considered in this paper.)
 Steps 
choosing individuals of other constraint 
cost 
values do not have any benefit for the analysis. Assuming that 
the successes increasing the 
quality and cost constraint value have 
the same success probability (or 
at least the same lower bound on it), 
the analysis essentially considers 
phases of uniform length where each phase 
must be concluded by a success 
for the next phase to start. 
More precisely, the runtime 
analysis typically allocates 
a window of~$t^*$ steps, for a certain 
number~$t^*$ depending on the 
success probability, and proves 
that a success is for within 
an expected number of~$t^*$ steps 
or is even highly likely in such
a phase. Nevertheless, although 
only steps choosing individuals of 
constraint value~$i$ are relevant 
for the phase, the classical GSEMO 
selects the individual 
to be mutated uniformly a random from the 
population,
leading to the ``wasted'' steps 
as mentioned above.

Our sliding window approach 
replaces the uniform selection 
with 
  a time-dependent window for selecting individuals  based on their constraint cost value. This time-dependent 
  window is of uniform length~$T$, which 
  is determined by the parameters of
  of the new algorithm, and 
  chooses individuals with constraint 
  cost value~$i$ for~$T$ steps each. 
  In our analysis, $T$ will be at least
  $t^*$, \ie, a proven length of 
  the phase allowing a success towards 
  a high-quality solution of the following constraint cost value 
  with high probability. 
   This allows the algorithm 
   to make time-dependent progress 
   and to focus the search 
   on solutions of a ``beneficial'' 
   cost constraint value, 
   which in the end results with high probability in the same approximation guarantee as the standard Pareto optimization approaches using 
   classical \gsemo. Our analysis points out that our proven runtime bound is independent of the maximum population size that the algorithm attains during its run. The provides significantly better upper bounds for our fast Pareto optimization approach compared to previous Pareto optimization approaches.

In our experimental study, we consider a wide range of social network graphs with different types of cost constraints and investigate the classical NP-hard maximum coverage problem. 
We consider uniform constraints where each node has a cost of $1$ as well as random constraints where the cost are chosen uniformly at random in the interval $[0.5, 1.5]$. Compared to previous studies, we examine settings that result in a much larger number of trade-offs as we investigate larger graphs in conjunction with larger constraint bound.
We point out that our sliding window selection provides significantly better results for various time budgets and constraint settings. Our analysis in terms of the resulting populations shows that the sliding window approach provides a much more focused Pareto optimization approach that produces a significantly larger number of trade-offs with respect to the constraint value and the considered goal function, in our case the maximum coverage score.

The outline of the paper is as follows. In Section~\ref{sec2}, we introduce the class of problems we consider in this paper. Section~\ref{sec3} presents our new \swgsemo algorithm. We provide our theoretical analysis of this algorithm in Section~\ref{sec4} and an experimental evaluation in Section~\ref{sec5}. Finally, we finish with some concluding remarks.

\section{Preliminaries}
\label{sec2}

In this section, we describe the problem classes and algorithms relevant for our study. 
Overall, the aim is to maximize pseudo-boolean fitness/objective 
functions $f\colon\{0,1\}^n\to \R$ under 
constraints on the allowed search points. An important class of such functions 
is given by the so-called \emph{submodular} functions \cite{NemhauserWolsey78}. We 
formulate it here on bit strings on length~$n$; in the literature, also an equivalent 
formulation using subsets of a ground set $V=\{v_1,\dots,v_n\}$ can be found. The notation 
$x\le y$ for two bit strings $x,y\in\{0,1\}^n$ means that $x$ is component-wise no larger than~$y$, 
\ie, $x_i\le y_i$ for all $i\in\{1,\dots,n\}$.

\begin{definition}
\label{def:submodular}
 A function $f\colon\{0,1\}^n\to\R^+$ is 
 called submodular if for all $x,y\in \{0,1\}^n$ where 
 $x\le y$ and all~$i$ where $x_i=y_i=0$
 it holds that 
 \[
f(x\oplus e_i)-f(x) \ge f(y\oplus e_i)-f(x)
 \]
 where $a\oplus e_i$ is the bit string 
 obtained by setting bit~$i$ of~$a$ to~$1$.
\end{definition}

We will also consider functions that are \emph{monotone} (either 
being submodular at the same time or without being submodular). A function 
$f\colon\{0,1\}^n\to \R$ is called monotone if for all $x,y\in\{0,1\}^n$
where $x\le y$ it holds that $f(x)\le f(y)$, \ie, setting bits to~$1$ 
without setting existing 1-bits to~$0$ will not decrease fitness.

Optimizing motonone functions becomes difficult if different constraints 
are introduced. Formally, there is a cost function $c\colon\{0,1\}^n\to\R$ and a budget~$B\in \R$ 
that the cost has to respect. Then the general optimization problem is defined as 
follows.

\begin{definition}[General Optimization Problem]
\label{def:genprob}
Given a monotone objective function $f\colon\{0,1\}^n\to \R$, a monotone 
cost function $c\colon\{0,1\}^n\to \R$ and a budget~$B\in \R$, 
find
\[\arg\max_{x\in\{0,1\}^n} f(x) \text{ such that }c(x)\le B.\]
\end{definition}

A constraint function~$c$ is called \emph{uniform} if it just counts 
the number of one-bits in~$x$, \ie, $c(x)=\sum_{i=1}^n x_i$.

In recent years, several variations of the problem in Definition~\ref{def:genprob} have been 
solved using multi-objective evolutionary algorithms like the classical GSEMO algorithm~\cite{FriedrichNeumannECJ15,DBLP:conf/nips/QianYZ15,QianSYTIJCAI17,DBLP:journals/ai/RoostapourNNF22}. 
In the following, we 
describe the basic concepts 
of multi-objective optimization relevant 
for our approach.


Let $x, y \in \{0,1\}^n$ be two search points. We say that $x$ (weakly) dominates y ($x \succeq y$) iff $f(x) \geq f(x)$ and $c(x) \leq c(y)$ holds.
We say $x$ strictly dominates $y$ ($x \succeq y$) iff $x \succeq y$ and $(f(x) \not = f(y) \vee c(x) \not = c(y))$. The dominance relationship applies in the same way to the objective vectors $(f(x), c(x))$ of the solutions. The set of non-dominated solutions is called the Pareto set and the set of non-dominated objectice vectors is called the Pareto front. 

The classical goal in multi-objective optimization is to compute for each Pareto optimal objective vector a corresponding solution. 
The approaches using Pareto optimization for constrained submodular problems
as investigated in, \eg,  \cite{FriedrichNeumannECJ15, QianSYTIJCAI17}, 
differ from this goal. Here the multi-objective approach is used to obtain a feasible solution, \ie, a solution for which $c(x) \leq B$ holds, that has the largest possible functions value $f(x)$.

\section{Sliding Window \gsemo}
\label{sec3}
The classical GSEMO algorithm~\cite{DBLP:journals/tec/LaumannsTZ04,1299908} (see Algorithm~\ref{alg:GSEMO}) has been widely used in the context of Pareto optimization. As done in~\cite{QianSYTIJCAI17}, we consider the variant starting with the search point $0^n$. This search point is crucial for its progress and the basis for obtaining theoretical performance guarantees.
\gsemo keeps for each non-dominated objective vector obtained during the optimization run exactly one solution in its current population $P$. In each iteration one solution $x \in P$ is chosen for mutation to produce an offspring $y$. The solution $y$ is accepted and included in the population if there is no solution $z$ in the current population that strictly dominates $y$. If $y$ is accepted, then all solutions that are (weakly) dominated by $y$ are removed from $P$.

We introduce the Sliding Window \gsemo (SW-GSEMO) algorithm given in Algorithm~\ref{alg:GSEMO-sliding}.
The algorithm differs from the classical \gsemo algorithm by selecting the parent $x$ that is used for mutation in a time dependent way with respect to its constraint value $c(x)$ (see the sliding-selection procedure given in Algorithm~\ref{alg:select-sliding}).
Let  $\temax$ be the total time that we allocate to the sliding window approach. 
Let $t$ the the current iteration number. If $t \leq  \temax$, the we select an individual of constraint value which matches the linear time progress from $0$ to constraint bound $B$, i.e. an individual with constraint value $\hat{c}=B/\temax$. As this value might not be integral,  we use the interval $[\lfloor \hat{c} \rfloor , \lceil \hat{c} \rceil ]$.
In the case that there is no such individual in the population, an individual is chosen uniformly at random from the whole population $P$ as done in the classical \gsemo algorithm.


 For mutation we analyze standard bit mutation. Here we create $y$ by flipping each bit $x_{i}$ of $x$ with probability $\frac{1}{n}$. As standard bit mutations have a probability of roughly $1/e$ of not flipping any bit in a mutation step, we use the standard bit mutation operator plus outlined in Algorithm~\ref{alg:mutation-plus} in our experimental studies. We note that all theorectical results obtained in this paper hold for standard bit mutations and standard bit mutatations plus.

As done in the area of runtime analysis, we measure the runtime of an algorithm buy the number of fitness evaluations. We analyze the Sliding Window \gsemo algorithm with respect to the number of fitness evaluations and determine values of $\temax$ for which the algorithm has obtained good approximation with high probability, i.e. with probability $1-o(1)$. The determined values of $\temax$ in our theorems in Section~\ref{sec4} are significantly lower than the bounds on the expected time to obtain the same approximations by the classical \gsemo algorithm.
\begin{algorithm}[t]
Set $x=0^n$\;
 $P\leftarrow \{x\}$\;
\Repeat{$\mathit{stop}$}{
Choose $x\in P$ uniformly at random\;
Create $y$ from $x$ by mutation\;
\If{$\nexists\, w \in P: w \succ y$} {
  $P \leftarrow (P \setminus \{z\in P \mid y \succeq z\}) \cup \{y\}$\;
  } }
\caption{Global simple evolutionary multi-objective optimizer (GSEMO)} \label{alg:GSEMO}
\end{algorithm}

\begin{algorithm}[t]
Set $x=0^n$\;
 $P\leftarrow \{x\}$\;
 $t \leftarrow 0$\;
\Repeat{$\mathit{t\geq \temax}$}{
$t \leftarrow t+1$\;
Choose $x=\text{sliding-selection}(P, t, \temax, B)$\;
Create $y$ from $x$ by mutation\;
\If{$\nexists\, w \in P: w \succ y$} {
  $P \leftarrow (P \setminus \{z\in P \mid y \succeq z\}) \cup \{y\}$\;
  } }
\caption{Sliding Window GSEMO (SW-GSEMO)} \label{alg:GSEMO-sliding}
\end{algorithm}

\begin{algorithm}[t]
\If{$t \leq \temax$}{
$\hat{c} \leftarrow (t/\temax)\cdot B$\;
$\hat{P} =\{x \in P \mid \lfloor \hat{c} \rfloor - \leq c(x) \leq \lceil \hat{c} \rceil\}$\;
\If{$\hat{P}=\emptyset$}{$\hat{P} \leftarrow P$\;}
}\Else{
$\hat{P} \leftarrow P$\;
}
Choose $x\in \hat{P}$ uniformly at random\;
Return $x$\;
\caption{sliding-selection$(P, t, \temax, B)$} \label{alg:select-sliding}
\end{algorithm}

\begin{algorithm}[t]
$y \leftarrow x$\;
\Repeat{$x \not =y$}{
Create $y$ from $x$ by flipping each bit $x_{i}$ of $x$ with probability $\frac{1}{n}$.
}
Return $y$\;
\caption{Standard-bit-mutation-plus(x)} \label{alg:mutation-plus}
\end{algorithm}

\section{Improved Runtime Bounds for SW-GSEMO}
\label{sec4}
Friedrich and Neumann~\cite{FriedrichNeumannECJ15} show that the classical GESMO 
finds a $(1-1/e)$-approximation to a monotone submodular function under a uniform 
constraint of size~$r$ in expected time $O(n^2(\log n + r))$. A factor~$n$ in their analysis stems 
from the fact that the population of GESMO can have size up to~$r\le n$. Thanks 
to the sliding window 
approach, which only selects 
from a certain subset of the population, this factor~$n$ 
does not appear in the following bound that we prove for \slidinggsemo. More precisely, the 
runtime guarantee in the following 
theorem~\ref{theo:sliding-uniform} is by a factor of $\Theta(n/\!\log n)$ better 
if $r\ge \log$. For smaller~$r$, we 
gain a factor of at least~$\Theta(n/r)$.

 \begin{theorem}
 \label{theo:sliding-uniform}
 Consider the \slidinggsemo with $\temax=4ern\ln n$ 
 on a monotone submodular function~$f$ under a uniform constraint of size~$r$. Then 
 with probability $1-o(1)$, the time until a $(1-1/e)$-approximation has been found is bounded 
 from above by $\temax = O(nr\log n)$.
 \end{theorem}

\begin{proof}
    We follow the proof of Theorem~2 in \cite{FriedrichNeumannECJ15}. As in that work, the aim 
    is to include for every $j\in\{0,\dots,r\}$ an element $x_j$ in the population such that 
    $$f(x_j)\ge (1-(1-1/r)^j)f(\OPT),$$ 
    where $\OPT$ is an optimal solution. 
    If this holds, then the element $x_r$ satisfies the 
    desired approximation ratio. We also know from \cite{FriedrichNeumannECJ15} that 
    the probability of mutating $x_j$ to $x_{j+1}$ is at least~$1/(en)$ since 
    it is sufficient to insert the element yielding the largest increase of~$f$ and not 
    to flip the rest of the bits.

    We now consider a sequence of events leading to the inclusion of elements~$x_j$ in the population 
    for growing~$j$. By definition of \slidinggsemo, element $x_0$ is in the population at 
    time~$0$. Assume that  
    element $x_j$, where $j\in\{0,\dots,r-1\}$, is in the population~$P$ at time $\tau_j \coloneqq 4ejn\ln n$. Then, by definition of 
    the set $\hat{P}$, it is available for selection 
    up to time $$\tau_{j+1}-1 = 4e(j+1)n\ln n-1$$ since 
    $$\lfloor ((4e(j+1)n\ln n-1)/\temax)\cdot r \rfloor= j.$$ 
    
    Moreover, the size of the subset population $\hat{P}$ that 
    the algorithm selects from is bounded from above by $2$ since $\lceil\hat{c}\rceil - \lfloor\hat{c}\rfloor \le 1$ 
    and there is 
    at most one non-dominated element for every constraint value. Therefore, the probability of choosing $x_j$ and mutating it 
    to $x_{j+1}$ is at least $1/(2en)$ from time~$\tau_j$ until time $\tau_{j+1}-1$, \ie, for a period 
    of $4en\ln n$ steps, and the probability 
    of this not happening is at most 
    $$(1-1/(2en))^{4en\ln n}\le 1/n^2.$$ By a union bound over the $r$ elements to be included, 
    the probability that there is a $j\in\{1,\dots,r\}$ such that 
    $x_j$ is not included by time $\tau_j$ is $O(1/n)$.
\end{proof}

 \subsection{General Cost Constraints}

 We will show that the sliding window approach can also be used 
 for more general optimization problems
 than submodular functions 
 and also leads to improved 
 runtime guarantees there. Specifically, 
 we extend the approach to 
 cover  general cost constraints 
similar to the scenario studied in 
 \cite{QianSYTIJCAI17}. They use 
 GSEMO (named POMC in their work) 
 for the optimization 
 scenario described in Definition~\ref{def:genprob}, \ie, 
  a monotone 
  objective function 
  $f\colon \{0,1\}^n\to \R^+$
  under the constraint that 
  a monotone cost function 
  $c\colon \{0,1\}^n \to \R^+$ 
  respects 
  a cost bound~$B$.

 When transferring the set-up 
 of \cite{QianSYTIJCAI17}, we 
 have introduce the following 
 restriction: the image set 
 of the cost function must be 
 the positive integers, \ie, 
 $c\colon \{0,1\}^n \to \N$. 
 Otherwise, all cost values could 
 be in a narrow real-valued interval 
 so that the sliding window 
 approach would not necessarily have 
 any effect.

 To formulate our result, we define 
 the submodularity ratio in the same 
 way as in earlier work (\eg, \cite{QianSYTIJCAI17}).

 \begin{definition}
 The submodularity ratio of 
 a function $f\colon\{0,1\}^n\to\R^+$ 
 with respect to a constraint 
 $c\colon\{0,1\}^n\to\N$ 
 is defined as
 \[
 \alpha_f\coloneqq \min_{\substack{x,y\in\{0,1\}^n, i\in\{1,\dots,n\}\\x\le y \wedge x_i=y_i=0}} \frac{f(x\oplus e_i)-f(x)}{c(y\oplus e_i)-c(y)},
 \]
 where $a\oplus e_i$ is the bit string 
 obtained by setting bit~$i$ of~$a$ to~$1$.\end{definition}

 The approximation guarantee proved 
 for GSEMO in \cite{QianSYTIJCAI17} 
 depends on $\alpha_f$ and the following 
 quantity. 

 \begin{definition}
     The minimum marginal gain of 
     a function $c \colon \{0,1\}^n\to \N$  
     is defined as
     \[\delta_{c} = 
     \min_{
     x\in\{0,1\}^n, i\in\{1,\dots,n\}, x_i=0} c(x\oplus e_i) - c(x)
     \]
 \end{definition}
 
In our analysis, we shall follow the assumption from \cite{QianSYTIJCAI17} 
that $\delta_c>0$, \ie, adding an element 
to the solution will always increase the cost value.
 
Theorem~2 in \cite{QianSYTIJCAI17} depends on the submodularity ratio, 
minimum marginal gain and the 
maximum population size $P_{\max}$ 
reached by a run of GSEMO on the 
bi-objective function maximizing 
$f$ and minimizing a variant~$\hat{c}$
of~$c$ (explained below).  
It states 
that within $O(en B P_{\max} /\delta_{\hat{c}})$ iterations, GESMO
finds a solution $x$ such that 
$f(x) \ge \frac{\alpha_f}{2} (1-1/e^{\alpha_f}) \cdot f(\hat{x})$, 
where $\hat{x}$ is an optimal solution 
when maximizing $f$ with the original
cost function~$c$ but under a slightly 
increased budget with 
respect to~$B$ (precisely 
defined in Equation~(4) in \cite{QianSYTIJCAI17} and further detailed 
in \cite{ZhangVorobeychickAAAI16}). Our main 
result, formulated 
in the following theorem, 
is that the \slidinggsemo 
obtains solutions with the same quality guarantee in an 
expected time where the $P_{\max}$ 
factor does not appear. As a detail in the formulation, the 
algorithm can be run with an approximation $\hat{c}$ of the original 
cost function $c$ that is by a certain factor $\phi(n)$ larger (\ie, 
$c(x)\le \hat{c}(x)\le \phi(n)c(x)$, see 
\cite{ZhangVorobeychickAAAI16} for details).

\begin{theorem}
\label{theo:general}
    Consider the problem of maximizing 
    a monotone function $f\colon\{0,1\}^n\to \R^+$ under 
    a monotone approximate cost function~$\hat{c}\colon\{0,1\}^n\to \N^+$ with 
    constraint~$B$ 
    and apply \slidinggsemo to the
    objective function $(f_1(x),f_2(x)$), where $f_2(x)=-\hat{c}(x)$ and \[f_1(x)=\begin{cases}
    -\infty & \text{ if $\hat{c}(x) > B$} \\f(x) & \text{otherwise}\end {cases}.\] If 
    $\temax=2en (B/\delta_{\hat{c}})\ln(n+B/\delta_{\hat{c}})$, then 
    with probability 
    at least~$1-o(1)$ a solution
    of quality at least 
    $\frac{\alpha_f}{2} (1-1/e^{\alpha_f}) \cdot f(\hat{x})$ is 
    found in $\temax$ iterations, 
    with $\delta_{\hat{c}}>0$, $\alpha_f$ and 
    $\hat{x}$ as defined above.
\end{theorem}

\begin{proof}
    We follow the proof of Theorem~2 in \cite{QianSYTIJCAI17} and adapt 
    it in a similar way to \slidinggsemo as we did in the proof of 
    Theorem~\ref{theo:sliding-uniform} above. According to the analysis 
    in \cite{QianSYTIJCAI17}, the approximation result is achieved 
    by a sequence of steps choosing an individual of cost value at most~$j$, 
    where $j\in\{0,\dots,B-1\}$ (here we adapted the proof to the integrality 
    of $\hat{c}$) and flipping a zero-bit to~$1$ yielding a certain minimum 
    increase of~$f$. More precisely, denoting by~$P$ the current population 
    of \slidinggsemo, we analyze the development of the 
    quantity
    \newcommand{\jmax}{J_{\max}}
    \begin{align*}
    \jmax =  \max\{j & \le B-1 \mid  \,\exists x\in P\colon 
    \hat{c}(x)\le j \\
    & \,\wedge\, f(x) \ge (1-(1- \tfrac{\alpha_f j}{Bk})^k)\cdot f(\hat{x}) 
    \text{\, for some~$k$}\}.
    \end{align*}
    Clearly, the initial value of~$\jmax$ is~$0$. According to the analysis 
    in~\cite{QianSYTIJCAI17}, if currently $\jmax=i<B$, then there is either  
    a solution~$x_i^*$ of cost $\hat{c}(x^*_i)\le i$ (note that the inequality 
    can be strict) and a zero-bit~$z_i^*$ in~$x$ whose 
    flip leads to $\jmax\ge i+\delta_{\hat{c}}$ (hereinafter called a successful 
    step); or there is already 
    a solution in the population satisfying the desired quality guarantee 
    $\frac{\alpha_f}{2} (1-1/e^{\alpha_f}) \cdot f(\hat{x})$.

    We now show that $x_i^*$ is chosen and bit~$z_i^*$ is flipped with high probability for each of at most $B/\delta_{\hat{c}}$ values that~$i$ can take
    in this sequence of successful steps. To this end, note that by definition, $\hat{c}(x_i^*)$ is increasing with respect to~$i$. Furthermore, by 
    definition of the set $\hat{P}$ in 
    \slidinggsemo, element $x_i^*$ from the population is available for selection 
    for a period of $$2en \ln(n+B/\delta_{\hat{c}})$$ steps, more precisely between time 
    $$2en \hat{c}(x_i^*) \ln(n+B/\delta_{\hat{c}})$$ and $$2en (\hat{c}(x_i^*)+1) \ln(n+B/\delta_{\hat{c}})-1.$$ 
    Moreover, 
    by the same arguments as in the proof of Theorem~\ref{theo:sliding-uniform}, 
    it holds that $\lvert \hat{P} \rvert\le 2$ for the subset population 
    $\hat{P}$ that 
    the algorithm selects from, so the probability of a success is at least 
    $1/(2en)$ for each stop within the mentioned period. Therefore, the probability of not having a
    successful step with respect to $x_i^*$ is bounded from above 
    by 
    $$(1-1/(2en))^{2en \ln(n+B/\delta_{\hat{c}})}\le \tfrac{1}{n(B/\delta_{\hat{c}}) }.$$ 
    By a union bound over the at most 
    $B/\delta_{\hat{c}}$ required successes, the probability of missing at least 
    one success is at most $1/n$, so altogether the desired solution quality 
    is achieved with probability at least $1-1/n=1-o(1)$.
\end{proof}

The results from Theorems~\ref{theo:sliding-uniform} and 
\ref{theo:general} are just two 
examples of a runtime result following 
an inductive sequence of improving steps 
based on constraint cost value. In 
the literature, there are further 
analyses of \gsemo
following a similar approach (\eg, \cite{QianLFTTCS23}),
which we believe can be transferred 
to our sliding window approach to 
yield improved runtime guarantees.

\section{Experimental investigations}
\label{sec5}

We now carry out experimental investigations of the new Sliding Window \gsemo approach and compare it against the standard \gsemo which has been used in many theoretical and experimental studies on submodular optimization.

\subsection{Experimental setup}
We consider the maximum coverage problem in graphs which is one of the best known NP-hard submmodular combinatorial optimization problems. Given an undirected graph $G=(V,E)$ with $n=|V|$ nodes, we denote by $N(v)$ the set of nodes containing $v$  and its neighbours in $G$. Each node $v$ has a cost $c(v)$ and the goal is to select a set of nodes indicated $x$ such that the number of nodes covered by~$x$ given as
\[\text{Coverage}(x) = \left|\bigcup_{i=1}^n N(v_i) x_i \right|
\]is maximized under the condition that the cost of the selected nodes does not exceed a given bound B, \ie
\[\sum_{i=1}^n c(v_i) x_i \leq B
\]
holds.

Previous studies investigated \gsemo either considered relatively small graphs or small budgets on the constraints that only allowed a small number of trade-offs to be produced. 
We consider sparse graphs from the from the network data repository~\cite{nr} as dense graphs are usually easy to cover with just a small number of nodes and do not pose a challenge when considering the maximum coverage problem.
We use the graphs ca-CSphd, ca-GrQc, Erdos992, ca-HepPh, ca-AstroPh, ca-CondMat, which consist of $1882$,  $4158$,  $6100$, $11204$, $17903$, $21363$ nodes, respectively.
Note, that especially the large graphs with more than $10000$ nodes are pushing the limits for both algorithms when considering the given fitness evaluation budgets.

To match our theoretical investigations for the uniform constraint, we consider the \emph{uniform} setting where $c(v_i)=1$ holds for any node~$v_i$ in the given graph. Note that the uniform setting implies that the number of trade-offs in the two objectives is at most $B+1$.
In order to investigate the behaviour of the \gsemo approaches when there are more possible trade-offs, we investigate for each graph the following random setting. For an instance in the \emph{random} setting, the cost of each node $v_i$ is chosen independently of the other uniformly at random in the interval $[0.5, 1.5]$. Note that the expected cost of each node in the random setting is $1$. This set up allows us to work with the same bounds for the uniform and random setting. We use 
$$B= \log_2 n, \sqrt{n}, \lfloor n/20 \rfloor, \lfloor n/10 \rfloor$$
such that the bounds scale with the given number of nodes in the considered graphs. Note that for the uniform case, all costs are integers and the effective bounds are the stated bounds rounded down to the next smaller integer. We consider the performance of both algorithms when given them a budget of
$$\temax=100000, 500000, 1000000$$ 
fitness evaluations.
Note, that especially for the large graphs and for the larger values of $B$ this is this pushing the limits as the fitness evaluations budgets are much less then what is stated in the upper bounds on $\temax$ in Theorems~\ref{theo:sliding-uniform} and~\ref{theo:general}. For every setting, we carry out 30 independent runs. We show the coverage values obtained in Table~\ref{tab:combined}. In this table, we report the mean, standard deviation and the $p$-value (with $3$ decimal places) obtained by the Mann Whitney\nobreakdash-U test. We call a result
statistically significant if the $p$-value is at most $0.05$. To gain additional insights, we present the mean final population sizes among the 30 runs for each setting in Table~\ref{tab:combinedpop}.

\begin{figure}[t]
\includegraphics[scale=0.5]{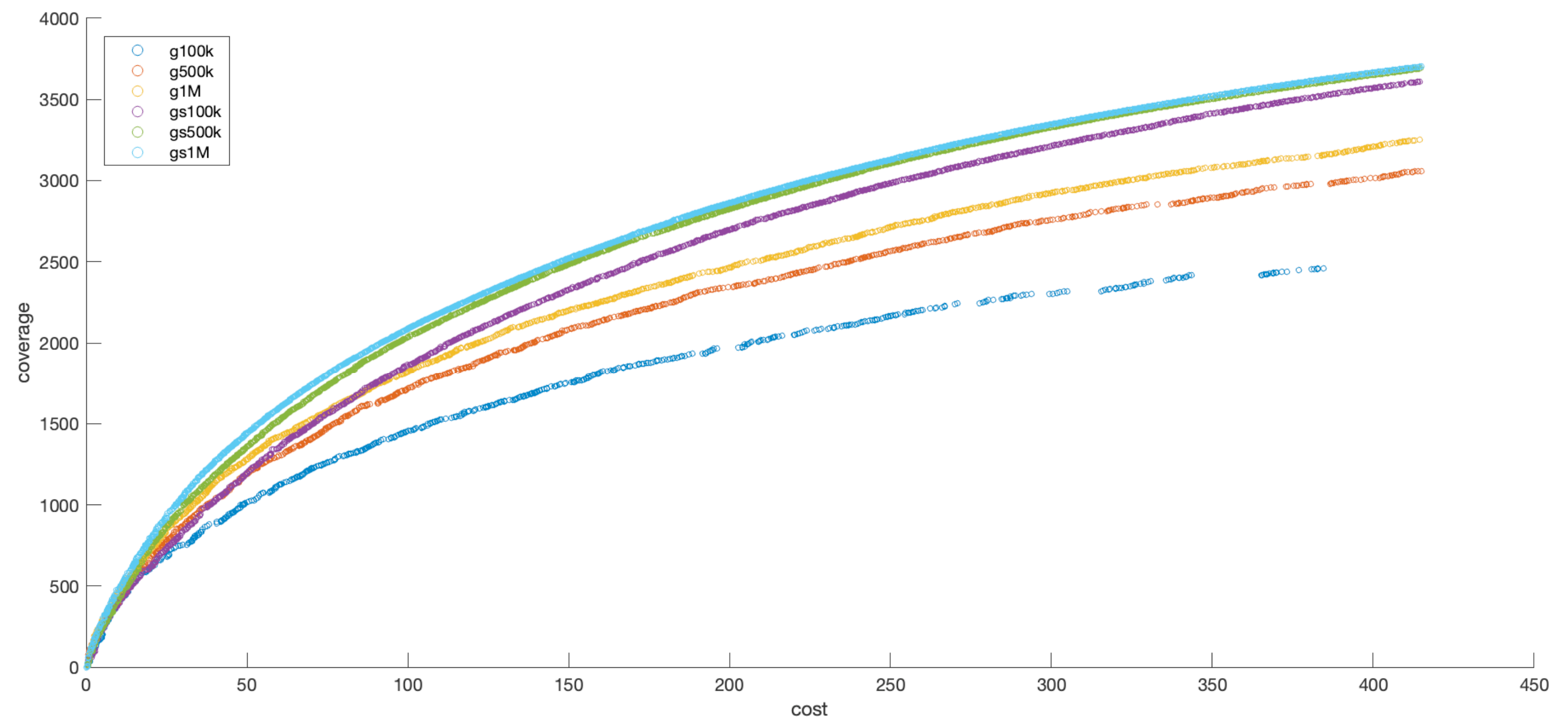}

\caption{Example final trade-offs for the graph ca-GrQc (4158 nodes) with randomly chosen costs and budget\\ $B=415$ obtained by \gsemo(g) and Sliding Window \gsemo(gs) for runs with $100k$, $500k$, and $1M$ iterations.}
\label{fig:ca-GrQc}
\end{figure}

\subsection{Experimental results}
The results for all considered graphs, constraint bounds, fitness evaluation budgets in the uniform and random cost setting are given in Table~\ref{tab:combined}. The best results are highlighted in bold.
Overall, \swgsemo is clearly outperforming \gsemo for almost all settings and almost all results are statistically significant. The only instances where both algorithms perform equally good are 
when considering the small constraint bounds of $\log_2 n$ or $\sqrt{n}$ for ca-CSphd and $\log_2 n$ for ca-GrQc, Erdos992, ca-HepPh. An important observation is that \swgsemo with only $100{,}000$ iterations is already much better than \gsemo with $1{,}000{,}000$ fitness evaluations if the constraint bound is not too small.

Considering the larger constraint bounds $\lfloor n/20 \rfloor, \lfloor n/10 \rfloor$ for the graphs, we can see that \swgsemo is significantly outperforming \gsemo. The difference in terms of coverage values between the two algorithms becomes even more pronounced when considering the random instances compared to the uniform ones. Considering the graph ca-HepPh, we can see that standard \gsemo obtains better results than \swgsemo when considering the constraint bound of $\log_2 n$ in the uniform setting while \swgsemo is significantly outperforming \gsemo in all other settings. Looking at the results for the two largest graphs ca-AstroPh and ca-CondMat, we can see that \swgsemo significantly outperforms \gsemo in all considered settings. 

Comparing the uniform and the random setting, we can see that the coverage values obtained by \gsemo in the random setting are much smaller than in the uniform setting when considering the three larger graphs ca-HepPh, ca-AstroPh, ca-CondMat. The only exception are some results for the small bound $B=\log_2 n$. 
We attribute this deterioration of \gsemo to the larger number of trade-offs per cost unit encountered in the optimization process which significantly slows down \gsemo making progress towards the constraint bound. In contrast to this, the coverage values obtained by \swgsemo in the uniform and random setting when considering $\temax=500000, 1000000$ are quite similar and sometimes higher for the random setting. This indicates that the sliding window selection keeps steady progress for the random setting and is not negatively impacted by the larger number of trade-offs in the random setting.

In order to better understand the performance difference between the 
algorithms, we examine the trade-offs produced by the algorithms.
Figure~\ref{fig:ca-GrQc} illustrates the final set of trade-offs for \gsemo and \swgsemo with respect to cost and coverage values for the different number of fitness evaluations considered. 
The three lower trade-off fronts depicted in blue, red, and yellow are obtaining running \gsemo with 100,000, 500,000, and 1,000,000 iterations. The three higher trade-off fronts shown in purple, green, and light blue have been obtained by \swgsemo in 100000, 500000, and 1000000 iterations.
It can be observed that the fronts obtained by \swgsemo are significantly better than the ones obtained by \gsemo. The fronts for \swgsemo with 500000 and 1000000 iterations are very similar while the results for $100000$ are already better than the ones obtained by \gsemo with 1000000 iterations. This matches the behaviour that can already be observed for many results shown in Table~\ref{tab:combined}. Furthermore, it can be seen that \gsemo has already difficulties obtaining a solution with cost close to the constraint bound when using the smallest considered budget of $100000$ fitness evaluations. 

\begin{table*}[htbp]
\scriptsize
    \centering
    \begin{tabular}{|c|c|c||c|c|c|c|c||c|c|c|c|c|} \hline 
       & & & \multicolumn{5}{|c||}{\bfseries Uniform} & \multicolumn{5}{|c|}{\bfseries Random}\\ \hline
 & & &      \multicolumn{2}{|c}{\bfseries \gsemo } & \multicolumn{2}{|c|}{\bfseries \swgsemo} & & \multicolumn{2}{|c|}{\bfseries \gsemo} & \multicolumn{2}{c|}{\bfseries \swgsemo} & \\ 
 Graph & $B$ & $t_{\max}$ &     Mean & Std & Mean & Std & $p$-value & Mean & Std & Mean & Std & $p$-value  \\
\hline
  \multirow{12}{*}{ca-CSphd} &  10 &  100000 &  \textbf{222} & 0.183 & \textbf{222} & 0.000 & 0.824 & 244 & 12.904 & \textbf{254} & 13.117 & 0.006\\
 &  10 &  500000 &  \textbf{222} & 0.000 & \textbf{222} & 0.000 & 1.000 & \textbf{257} & 13.962 & \textbf{257} & 13.672 & 0.971\\
 &  10 &  1000000 &  \textbf{222} & 0.000 & \textbf{222} & 0.000 & 1.000 & \textbf{258} & 13.938 & 257 & 13.878 & 0.894\\
 &  43 &  100000 &  568 & 5.756 & \textbf{599} & 0.730 & 0.000 & 539 & 13.808 & \textbf{625} & 13.711 & 0.000\\
 &  43 &  500000 &  \textbf{600} & 0.254 & \textbf{600} & 0.000 & 0.657 & 615 & 13.150 & \textbf{629} & 13.485 & 0.000\\
 &  43 &  1000000 &  \textbf{600} & 0.000 & \textbf{600} & 0.000 & 1.000 & 626 & 13.588 & \textbf{630} & 13.688 & 0.234\\
 &  94 &  100000 &  823 & 6.150 & \textbf{928} & 0.430 & 0.000 & 779 & 12.898 & \textbf{957} & 12.038 & 0.000\\
 &  94 &  500000 &  924 & 1.570 & \textbf{928} & 0.000 & 0.000 & 915 & 12.252 & \textbf{962} & 12.159 & 0.000\\
 &  94 &  1000000 &  \textbf{928} & 0.254 & \textbf{928} & 0.000 & 0.657 & 946 & 11.660 & \textbf{963} & 12.173 & 0.000\\
 &  188 &  100000 &  1087 & 11.676 & \textbf{1280} & 0.814 & 0.000 & 1036 & 12.868 & \textbf{1334} & 12.709 & 0.000\\
 &  188 &  500000 &  1256 & 2.809 & \textbf{1280} & 0.000 & 0.000 & 1235 & 12.964 & \textbf{1341} & 12.718 & 0.000\\
 &  188 &  1000000 &  1278 & 1.119 & \textbf{1280} & 0.254 & 0.000 & 1291 & 11.671 & \textbf{1341} & 12.822 & 0.000\\
\hline
 \multirow{12}{*}{ca-GrQc} &  12 &  100000 &  490 & 8.798 & \textbf{505} & 5.701 & 0.000 & 535 & 17.964 & \textbf{595} & 21.765 & 0.000\\
 &  12 &  500000 &  509 & 2.539 & \textbf{510} & 0.000 & 0.046 & 601 & 23.067 & \textbf{620} & 20.341 & 0.003\\
 &  12 &  1000000 &  \textbf{510} & 0.000 & \textbf{510} & 0.000 & 1.000 & 615 & 22.523 & \textbf{623} & 22.192 & 0.181\\
 &  64 &  100000 &  1320 & 15.662 & \textbf{1511} & 5.488 & 0.000 & 1281 & 24.326 & \textbf{1636} & 25.545 & 0.000\\
 &  64 &  500000 &  1490 & 8.904 & \textbf{1529} & 3.319 & 0.000 & 1516 & 25.832 & \textbf{1692} & 25.286 & 0.000\\
 &  64 &  1000000 &  1512 & 6.244 & \textbf{1529} & 2.087 & 0.000 & 1594 & 24.659 & \textbf{1699} & 26.351 & 0.000\\
 &  207 &  100000 &  2151 & 20.651 & \textbf{2748} & 10.797 & 0.000 & 2044 & 25.039 & \textbf{2840} & 23.631 & 0.000\\
 &  207 &  500000 &  2530 & 13.505 & \textbf{2775} & 3.937 & 0.000 & 2450 & 18.991 & \textbf{2918} & 18.449 & 0.000\\
 &  207 &  1000000 &  2655 & 9.295 & \textbf{2778} & 2.687 & 0.000 & 2600 & 13.956 & \textbf{2926} & 20.322 & 0.000\\
 &  415 &  100000 &  2704 & 24.887 & \textbf{3571} & 5.661 & 0.000 & 2407 & 51.554 & \textbf{3622} & 15.866 & 0.000\\
 &  415 &  500000 &  3171 & 13.424 & \textbf{3616} & 3.336 & 0.000 & 3047 & 19.085 & \textbf{3701} & 14.015 & 0.000\\
 &  415 &  1000000 &  3339 & 10.258 & \textbf{3622} & 3.468 & 0.000 & 3232 & 13.908 & \textbf{3710} & 12.832 & 0.000\\
\hline
 \multirow{12}{*}{Erdos992} &  12 &  100000 &  584 & 8.094 & \textbf{601} & 2.102 & 0.000 & 635 & 26.211 & \textbf{750} & 37.157 & 0.000\\
 &  12 &  500000 &  603 & 1.837 & \textbf{604} & 0.183 & 0.012 & 751 & 36.183 & \textbf{783} & 37.723 & 0.002\\
 &  12 &  1000000 &  \textbf{604} & 0.254 & \textbf{604} & 0.000 & 0.657 & 774 & 37.049 & \textbf{786} & 38.149 & 0.191\\
 &  78 &  100000 &  1835 & 35.766 & \textbf{2453} & 6.674 & 0.000 & 1658 & 31.957 & \textbf{2512} & 44.639 & 0.000\\
 &  78 &  500000 &  2345 & 18.151 & \textbf{2472} & 0.791 & 0.000 & 2169 & 35.736 & \textbf{2626} & 44.925 & 0.000\\
 &  78 &  1000000 &  2438 & 7.439 & \textbf{2473} & 0.430 & 0.000 & 2366 & 41.017 & \textbf{2634} & 43.671 & 0.000\\
 &  305 &  100000 &  2862 & 54.832 & \textbf{4706} & 7.897 & 0.000 & 2547 & 60.705 & \textbf{4584} & 87.698 & 0.000\\
 &  305 &  500000 &  3824 & 27.333 & \textbf{4772} & 1.570 & 0.000 & 3534 & 30.716 & \textbf{4789} & 20.565 & 0.000\\
 &  305 &  1000000 &  4201 & 20.239 & \textbf{4775} & 0.765 & 0.000 & 3898 & 28.104 & \textbf{4799} & 21.492 & 0.000\\
 &  610 &  100000 &  3076 & 46.668 & \textbf{5251} & 2.343 & 0.000 & 2553 & 57.540 & \textbf{4734} & 464.930 & 0.000\\
 &  610 &  500000 &  4428 & 31.694 & \textbf{5263} & 0.964 & 0.000 & 4017 & 50.324 & \textbf{5376} & 8.230 & 0.000\\
 &  610 &  1000000 &  4791 & 22.216 & \textbf{5264} & 0.254 & 0.000 & 4516 & 33.380 & \textbf{5378} & 8.477 & 0.000\\
\hline
 \multirow{12}{*}{ca-HepPh} &  13 &  100000 &  1687 & 34.310 & \textbf{1806} & 30.914 & 0.000 & 1750 & 53.601 & \textbf{1979} & 53.815 & 0.000\\
 &  13 &  500000 &  1830 & 27.572 & \textbf{1844} & 11.761 & 0.019 & 1980 & 50.385 & \textbf{2115} & 48.030 & 0.000\\
 &  13 &  1000000 &  \textbf{1854} & 15.539 & 1840 & 5.112 & 0.000 & 2058 & 50.351 & \textbf{2144} & 54.399 & 0.000\\
 &  105 &  100000 &  3740 & 48.996 & \textbf{4644} & 27.022 & 0.000 & 3598 & 31.674 & \textbf{4806} & 52.069 & 0.000\\
 &  105 &  500000 &  4309 & 28.317 & \textbf{4789} & 13.226 & 0.000 & 4238 & 40.329 & \textbf{5111} & 53.447 & 0.000\\
 &  105 &  1000000 &  4516 & 20.372 & \textbf{4802} & 10.016 & 0.000 & 4490 & 43.562 & \textbf{5176} & 42.558 & 0.000\\
 &  560 &  100000 &  5909 & 73.941 & \textbf{8544} & 22.516 & 0.000 & 5108 & 86.715 & \textbf{8626} & 40.014 & 0.000\\
 &  560 &  500000 &  7092 & 29.683 & \textbf{8800} & 11.032 & 0.000 & 6802 & 36.876 & \textbf{9058} & 31.968 & 0.000\\
 &  560 &  1000000 &  7527 & 24.174 & \textbf{8825} & 7.388 & 0.000 & 7249 & 32.562 & \textbf{9129} & 28.224 & 0.000\\
 &  1120 &  100000 &  5919 & 81.766 & \textbf{10196} & 144.717 & 0.000 & 5098 & 92.907 & \textbf{9095} & 1221.796 & 0.000\\
 &  1120 &  500000 &  8148 & 55.111 & \textbf{10494} & 8.483 & 0.000 & 7209 & 62.132 & \textbf{10622} & 45.053 & 0.000\\
 &  1120 &  1000000 &  8795 & 30.765 & \textbf{10523} & 7.263 & 0.000 & 8138 & 62.450 & \textbf{10683} & 30.947 & 0.000\\
\hline
 \multirow{12}{*}{ca-AstroPh} &  14 &  100000 &  2594 & 85.426 & \textbf{2867} & 57.047 & 0.000 & 2598 & 96.079 & \textbf{3026} & 106.242 & 0.000\\
 &  14 &  500000 &  2914 & 31.062 & \textbf{2974} & 5.396 & 0.000 & 3016 & 64.712 & \textbf{3321} & 80.235 & 0.000\\
 &  14 &  1000000 &  2962 & 12.480 & \textbf{2980} & 2.716 & 0.000 & 3195 & 77.509 & \textbf{3388} & 87.267 & 0.000\\
 &  133 &  100000 &  6484 & 77.132 & \textbf{8351} & 50.919 & 0.000 & 6221 & 82.193 & \textbf{8558} & 79.255 & 0.000\\
 &  133 &  500000 &  7551 & 51.042 & \textbf{8709} & 25.847 & 0.000 & 7362 & 61.676 & \textbf{9214} & 64.955 & 0.000\\
 &  133 &  1000000 &  7968 & 40.202 & \textbf{8749} & 14.914 & 0.000 & 7817 & 62.797 & \textbf{9372} & 68.257 & 0.000\\
 &  895 &  100000 &  9511 & 120.361 & \textbf{15033} & 32.462 & 0.000 & 8170 & 122.610 & \textbf{14467} & 1262.081 & 0.000\\
 &  895 &  500000 &  12360 & 60.272 & \textbf{15611} & 16.028 & 0.000 & 11387 & 100.119 & \textbf{15861} & 34.325 & 0.000\\
 &  895 &  1000000 &  13017 & 35.102 & \textbf{15690} & 11.610 & 0.000 & 12490 & 39.435 & \textbf{16014} & 27.498 & 0.000\\
 &  1790 &  100000 &  9502 & 97.540 & \textbf{16984} & 162.622 & 0.000 & 8137 & 104.415 & \textbf{13313} & 2546.964 & 0.000\\
 &  1790 &  500000 &  12750 & 88.401 & \textbf{17473} & 8.353 & 0.000 & 11374 & 104.602 & \textbf{17039} & 822.188 & 0.000\\
 &  1790 &  1000000 &  14103 & 54.069 & \textbf{17527} & 6.369 & 0.000 & 12743 & 112.303 & \textbf{17543} & 158.486 & 0.000\\
\hline
 \multirow{12}{*}{ca-CondMat} &  14 &  100000 &  1514 & 56.623 & \textbf{1766} & 44.191 & 0.000 & 1411 & 76.166 & \textbf{1759} & 63.903 & 0.000\\
 &  14 &  500000 &  1802 & 25.451 & \textbf{1854} & 5.063 & 0.000 & 1773 & 62.028 & \textbf{2016} & 79.008 & 0.000\\
 &  14 &  1000000 &  1846 & 8.628 & \textbf{1857} & 1.717 & 0.000 & 1912 & 69.264 & \textbf{2068} & 77.399 & 0.000\\
 &  146 &  100000 &  4388 & 71.953 & \textbf{6668} & 49.261 & 0.000 & 4106 & 91.413 & \textbf{6650} & 68.737 & 0.000\\
 &  146 &  500000 &  5585 & 61.992 & \textbf{7054} & 17.553 & 0.000 & 5264 & 65.588 & \textbf{7412} & 73.412 & 0.000\\
 &  146 &  1000000 &  6092 & 50.561 & \textbf{7091} & 8.705 & 0.000 & 5776 & 64.557 & \textbf{7560} & 73.139 & 0.000\\
 &  1068 &  100000 &  7187 & 130.346 & \textbf{15758} & 38.782 & 0.000 & 5779 & 149.352 & \textbf{14515} & 2634.382 & 0.000\\
 &  1068 &  500000 &  11334 & 67.041 & \textbf{16727} & 18.791 & 0.000 & 9655 & 129.937 & \textbf{17038} & 47.148 & 0.000\\
 &  1068 &  1000000 &  12364 & 69.321 & \textbf{16844} & 11.085 & 0.000 & 11533 & 79.858 & \textbf{17287} & 53.123 & 0.000\\
 &  2136 &  100000 &  7211 & 133.166 & \textbf{19120} & 204.254 & 0.000 & 5823 & 137.877 & \textbf{13642} & 3603.578 & 0.000\\
 &  2136 &  500000 &  11556 & 115.359 & \textbf{20093} & 12.913 & 0.000 & 9675 & 135.492 & \textbf{19513} & 1682.649 & 0.000\\
  &  2136 &  1000000 &  13652 & 84.783 & \textbf{20218} & 9.453 & 0.000 & 11632 & 96.193 & \textbf{20446} & 87.106 & 0.000\\
 \hline
    \end{tabular}
    \caption{Maximum coverage scores obtained by \gsemo and \swgsemo
    }
    \label{tab:combined}
\end{table*}

In  Table~\ref{tab:combinedpop}, we show the average number of trade-offs given by the final populations of the two algorithms for the 30 runs of each setting. We first examine the uniform setting. 
The budgets for our experiments are chosen small enough such that not all nodes can be covered by any solution.
Therefore, in the ideal case, both algorithms would produce $B+1$ trade-offs in the uniform setting. It can be observed that this is roughly happening for the two smallest graphs ca-CSphd, ca-GrQc. For the remaining $4$ graphs, \gsemo produces significantly less points when considering the constraint bound $\lfloor n/20 \rfloor, \lfloor n/10 \rfloor$ while the number of trade-offs obtained by \swgsemo is close to $B+1$ is most uniform settings. Considering the random setting, we can see that the number of trade-offs produced by \swgsemo is significantly higher than for \gsemo. In the case of large graphs, the number of trade-offs produced is up to four times larger than the number of trade-offs produced by \gsemo, e.g. for graph ca-CondMat and $B=2136$. Overall this suggests that the larger number of trade-offs produced in a systematic way by the sliding window approach significantly contributes to the superior performance of \swgsemo.

 \begin{table}
\scriptsize
    \centering
    \begin{tabular}{|c|c|c||c|c||c|c|} \hline 
       & & & \multicolumn{2}{|c||}{\bfseries Uniform } & \multicolumn{2}{|c|}{\bfseries Random }\\ \hline
 
 Graph & $B$ & $t_{\max}$ &     G &  SWG & G & SWG   \\
\hline
 \multirow{12}{*}{ca-CSphd} &  10 &  100000 &  11 & 11 & 73 & 92\\
 &  10 &  500000 &  11 & 11 & 125 & 130\\
 &  10 &  1000000 &  11 & 11 & 133 & 132\\
 &  43 &  100000 &  44 & 44 & 172 & 280\\
 &  43 &  500000 &  44 & 44 & 287 & 435\\
 &  43 &  1000000 &  44 & 44 & 390 & 475\\
 &  94 &  100000 &  94 & 95 & 281 & 496\\
 &  94 &  500000 &  95 & 95 & 422 & 686\\
 &  94 &  1000000 &  95 & 95 & 540 & 761\\
 &  188 &  100000 &  180 & 189 & 404 & 785\\
 &  188 &  500000 &  189 & 189 & 591 & 985\\
 &  188 &  1000000 &  189 & 189 & 734 & 1068\\
\hline
 \multirow{12}{*}{ca-GrQc} &  12 &  100000 &  13 & 13 & 89 & 115\\
 &  12 &  500000 &  13 & 13 & 142 & 191\\
 &  12 &  1000000 &  13 & 13 & 185 & 236\\
 &  64 &  100000 &  65 & 65 & 261 & 439\\
 &  64 &  500000 &  65 & 65 & 372 & 700\\
 &  64 &  1000000 &  65 & 65 & 448 & 854\\
 &  207 &  100000 &  200 & 208 & 493 & 1021\\
 &  207 &  500000 &  207 & 208 & 710 & 1460\\
 &  207 &  1000000 &  208 & 208 & 841 & 1703\\
 &  415 &  100000 &  347 & 414 & 611 & 1500\\
 &  415 &  500000 &  400 & 416 & 967 & 1996\\
 &  415 &  1000000 &  409 & 416 & 1134 & 2246\\
\hline
 \multirow{12}{*}{Erdos992} &  12 &  100000 &  13 & 13 & 76 & 103\\
 &  12 &  500000 &  13 & 13 & 129 & 193\\
 &  12 &  1000000 &  13 & 13 & 181 & 251\\
 &  78 &  100000 &  78 & 79 & 237 & 395\\
 &  78 &  500000 &  79 & 79 & 330 & 684\\
 &  78 &  1000000 &  79 & 79 & 393 & 913\\
 &  305 &  100000 &  253 & 305 & 441 & 970\\
 &  305 &  500000 &  291 & 306 & 668 & 1497\\
 &  305 &  1000000 &  298 & 306 & 777 & 1885\\
 &  610 &  100000 &  296 & 588 & 440 & 1031\\
 &  610 &  500000 &  483 & 610 & 818 & 1769\\
 &  610 &  1000000 &  522 & 611 & 1014 & 2073\\
\hline
 \multirow{12}{*}{ca-HepPh} &  13 &  100000 &  14 & 14 & 91 & 118\\
 &  13 &  500000 &  14 & 14 & 125 & 157\\
 &  13 &  1000000 &  14 & 14 & 143 & 195\\
 &  105 &  100000 &  105 & 106 & 344 & 634\\
 &  105 &  500000 &  106 & 106 & 489 & 878\\
 &  105 &  1000000 &  106 & 106 & 553 & 1040\\
 &  560 &  100000 &  408 & 558 & 634 & 2083\\
 &  560 &  500000 &  521 & 560 & 1176 & 2849\\
 &  560 &  1000000 &  543 & 561 & 1364 & 3304\\
 &  1120 &  100000 &  409 & 1093 & 644 & 2295\\
 &  1120 &  500000 &  813 & 1116 & 1306 & 3854\\
 &  1120 &  1000000 &  951 & 1117 & 1718 & 4295\\
\hline
 \multirow{12}{*}{ca-AstroPh} &  14 &  100000 &  15 & 15 & 94 & 119\\
 &  14 &  500000 &  15 & 15 & 125 & 155\\
 &  14 &  1000000 &  15 & 15 & 142 & 190\\
 &  133 &  100000 &  132 & 134 & 404 & 808\\
 &  133 &  500000 &  134 & 134 & 565 & 1074\\
 &  133 &  1000000 &  134 & 134 & 663 & 1292\\
 &  895 &  100000 &  416 & 890 & 653 & 2719\\
 &  895 &  500000 &  764 & 895 & 1372 & 3921\\
 &  895 &  1000000 &  818 & 896 & 1755 & 4434\\
 &  1790 &  100000 &  413 & 1715 & 643 & 2237\\
 &  1790 &  500000 &  848 & 1767 & 1366 & 4525\\
 &  1790 &  1000000 &  1121 & 1776 & 1830 & 5263\\
\hline
 \multirow{12}{*}{ca-CondMat} &  14 &  100000 &  15 & 15 & 87 & 106\\
 &  14 &  500000 &  15 & 15 & 111 & 137\\
 &  14 &  1000000 &  15 & 15 & 124 & 166\\
 &  146 &  100000 &  144 & 147 & 410 & 819\\
 &  146 &  500000 &  147 & 147 & 572 & 1066\\
 &  146 &  1000000 &  147 & 147 & 662 & 1286\\
 &  1068 &  100000 &  424 & 1063 & 650 & 3244\\
 &  1068 &  500000 &  864 & 1068 & 1437 & 4953\\
 &  1068 &  1000000 &  952 & 1068 & 1929 & 5772\\
 &  2136 &  100000 &  425 & 2084 & 655 & 2801\\
 &  2136 &  500000 &  906 & 2125 & 1428 & 6313\\
  &  2136 &  1000000 &  1228 & 2131 & 1953 & 7724\\
\hline
    \end{tabular}
    \vspace{0.3cm}
    \caption{Final number of trade-off solutions obtained by \gsemo(G) and \swgsemo(SWG)}
    
    \label{tab:combinedpop}
\end{table}

 \section{Conclusions}
 Pareto optimization using \gsemo has widely been applied in the context of submodular optimization. We introduced the Sliding Window \gsemo algorithm which selects an individual due to time progress and constraint value in the parent selection step. Our theoretical analysis provides better runtime bounds for \swgsemo while achieving the same worst-case approxmation ratios as \gsemo. 
 Our experimental investigations for the maximum coverage problem shows that \swgsemo outperforms \gsemo for a wide range of settings. We also provided additional insights into the optimization process by showing that \swgsemo computes significantly more trade-off then \gsemo for instances with random weights or uniform instances with large budgets.

\section*{Acknowledgments}
This work has been supported by the Australian Research Council (ARC) through grant FT200100536.

\end{document}